\documentclass{article}

\PassOptionsToPackage{numbers}{natbib}
\usepackage[final]{nips_2017}

\usepackage[utf8]{inputenc} 
\usepackage[T1]{fontenc}    
\usepackage{booktabs}       
\usepackage{microtype}      
\usepackage{empheq}
\usepackage{setspace}
\usepackage{cprotect}
\usepackage{amsmath,amssymb}
\usepackage{amsthm}
\usepackage[noend]{algorithmic}
\usepackage[ruled,vlined]{algorithm2e}
\usepackage{url}
\usepackage{makeidx}
\usepackage{enumerate}
\usepackage{graphicx,float,psfrag,epsfig,caption,subcaption,enumitem}
\usepackage{epstopdf,hyperref}
\usepackage{color,xcolor}
\hypersetup{
  colorlinks,
  linkcolor={red!50!black},
  citecolor={blue!50!black},
  urlcolor={blue!80!black}
}

\usepackage[mathscr]{euscript}

\DeclareSymbolFont{rsfs}{U}{rsfs}{m}{n}
\DeclareSymbolFontAlphabet{\mathscrsfs}{rsfs}

\numberwithin{equation}{section}

\newtheoremstyle{myexample} 
    {\topsep}                    
    {\topsep}                    
    {\rm }                   
    {}                           
    {\bf }                   
    {.}                          
    {.5em}                       
    {}  

\newtheoremstyle{myremark} 
    {\topsep}                    
    {\topsep}                    
    {\rm}                        
    {}                           
    {\bf}                        
    {.}                          
    {.5em}                       
    {}  

\newtheorem{theorem}{Theorem}

\theoremstyle{myremark}

\theoremstyle{myremark}

\theoremstyle{myexample}

\def\<{\langle}
\def\>{\rangle}

\def\bb{{\boldsymbol b}}

\def\sign{{\operatorname{\text{sign}}}}

\def\normal{{\sf N}}

\def\cC{{\cal C}}
\def\cS{{\cal S}}

\def\cL{{\cal L}}

\def\K{{\cal K}}
\def\cR{\mathcal{R}}

\newcommand\norm[1]{\lVert{#1}\rVert}

\def\cutp{\mathcal{C}}

\def\metp{\mathcal{M}}
\def\deq{{\ \coloneqq}}

\newcommand{\eq}[1]{\begin{align}#1\end{align}}

\newcommand{\inner}[1]{{\langle #1 \rangle}}
\newcommand{\abs}[1]{\left|#1\right|}

\def\syms{{\mathbb S}}

\def\bb{{\bf b}}

\def\cX{{\cal X}}

\def\mge{{\, \succcurlyeq \,}}

\def\reals{\mathbb{R}}

\def\by{{\boldsymbol y}}

\def\bx{{\boldsymbol x}}

\def\btheta{{\boldsymbol \theta}}
\def\hbx{\hat{\boldsymbol x}}

\def\bOmega{{\boldsymbol \Omega}}

\def\bc{{\boldsymbol c}}
\def\bb{{\boldsymbol b}}

\def\bA{{\boldsymbol A}}

\def\blambda{{\boldsymbol \lambda}}

\def\s{{\boldsymbol \sigma}}
\def\bsigma{{\boldsymbol \sigma}}

\def\Treg[#1]{T^{{\rm reg},#1}}
\def\GW[#1]{{\rm GW}(#1)}
\def\MGW[#1]{{\rm MGW}(#1)}

\def\opt{\text INT}

\def\O{{\mathcal O}}
\def\te{\theta^{{\rm e}}}
\def\tv{\theta^{{\rm v}}}
\def\bte{\theta^{{\rm e}}}
\def\btv{\theta^{{\rm v}}}
\def\SOS{{\rm SOS}}
\def\PSOS{{\rm PSOS}}
\def\M{{\sf M}}
\def\tol{\text{tol}}

\def\Confidence{{\sf Confidence}}

\def\spherec{\bigcirc \!\!\!\!\raisebox{-.1pt}{\scriptsize $i$}}

\title{Inference in Graphical Models \\ via Semidefinite Programming Hierarchies}

\author{
  Murat A.~Erdogdu\\
  Microsoft Research\\
  \!\!\!\!\texttt{erdogdu@cs.toronto.edu} \!\!\!\! \\
  \And
   Yash Deshpande \\
   MIT and Microsoft Research\\
  \texttt{yash@mit.edu} \\
   \And
  Andrea Montanari \\
     Stanford University\\
  \texttt{montanari@stanford.edu} \\
}

\begin{document}

\maketitle

\begin{abstract}
  Maximum A posteriori Probability (MAP) inference in graphical models
  amounts to solving a graph-structured  combinatorial optimization problem.
  Popular inference algorithms such as belief propagation (BP) and
  generalized belief propagation (GBP) are intimately related to
  linear programming (LP) relaxation  within the Sherali-Adams hierarchy.
  Despite the popularity of these algorithms,  it is well understood that
  the Sum-of-Squares (SOS) hierarchy based on
  semidefinite programming (SDP) can provide superior guarantees.
  Unfortunately, SOS relaxations
  for a graph with $n$ vertices require solving an SDP
  with $n^{\Theta(d)}$ variables where $d$ is the degree in the hierarchy.
  In practice, for $d\ge 4$, this approach does not scale beyond a few tens of variables. 
  In this paper, we propose binary SDP relaxations for MAP inference using the SOS hierarchy
  with two innovations focused on computational efficiency. Firstly, in analogy to
  BP and its variants, we only introduce decision variables corresponding to contiguous regions
  in the graphical model. Secondly, we solve the resulting SDP using a non-convex Burer-Monteiro style method,
  and develop a sequential rounding procedure.
  We demonstrate that the resulting algorithm can solve problems
  with tens of thousands of variables within minutes, and outperforms
  BP and GBP on practical problems such as image denoising and Ising spin glasses.
  Finally, for specific graph types, we establish a sufficient condition for the
  tightness of the proposed partial SOS relaxation.
\end{abstract}

\vspace{-.1in}
\section{Introduction}
\vspace{-.1in}
Graphical models provide a powerful framework for analyzing systems
comprised by a large number of interacting variables.
Inference in graphical models is crucial
in scientific methodology with countless applications
in a variety of fields including causal inference, computer vision, statistical physics,
information theory, and genome research
\cite{wainwright2008graphical,koller2009probabilistic,mezard2009information}.

In this paper, we propose a class of inference algorithms
for pairwise undirected graphical models. Such models are
fully specified by assigning:
$(i)$ a finite domain $\cX$ for the variables;
$(ii)$ a finite graph $G=(V,E)$ for $V=[n]\equiv\{1,\dots,n\}$
capturing the interactions of the basic variables;
$(iii)$ a collection of functions 
$\btheta = (\{\theta_i^v\}_{i\in V}, \{\theta_{ij}^e\}_{(i,j)\in E})$
that quantify the vertex potentials and interactions between the variables;
whereby for each vertex $i\in V$ we have $\theta_i^v:\cX\to \reals$
and for each edge
$(i,j)\in E$, we have
$\theta_{ij}^e:\cX\times\cX\to \reals$ (an arbitrary ordering is fixed
on the pair of vertices $\{i,j\}$).
These parameters can be used to form a probability distribution on
$\cX^V$ for the random vector $\bx = (x_1,x_2, ...,x_{n})\in\cX^V$ by letting,
\begin{align}
  p(\bx|\btheta) = \frac{1}{Z(\btheta)}\, e^{U(\bx;\btheta)}\, ,\;\;\;\; \ 
  U(\bx;\btheta) = \sum_{(i,j)\in E}\theta_{ij}^e(x_i,x_j)+\sum_{i\in V}\theta_i^v(x_i)\, ,
\end{align}
where $Z(\btheta)$ is the normalization constant commonly referred to
as the partition function.
While such models can encode a rich class of multivariate
probability distributions, basic inference tasks are intractable except for
very special graph structures such as trees or
small treewidth graphs \cite{cowell2006probabilistic}.
In this paper, we will focus on MAP estimation, which amounts to 
solving the combinatorial optimization problem
\vspace{-.05in}
\begin{align}
  \hbx(\btheta) \equiv \arg \max_{\bx\in\cX^V} U(\bx;\btheta).
\end{align}
Intractability plagues other classes of graphical models as well
(e.g. Bayesian networks, factor graphs), and has motivated 
the development of a wide array of heuristics.
One of the simplest such heuristics is the loopy belief propagation (BP)
\cite{wainwright2008graphical,koller2009probabilistic,mezard2009information}.
In its max-product version (that is well-suited for MAP estimation),
BP is intimately related to the linear programming (LP) relaxation
of the combinatorial problem $\max_{\bx\in\cX^V}U(\bx;\btheta)$. 
Denoting the decision variables by
$\bb = (\{b_i\}_{i\in V}, \{b_{ij}\}_{(i,j)\in E})$,
LP relaxation form of BP can be written as
\begin{align}
  \underset{\bb}{\mbox{maximize}}&\;\;\; \sum_{(i,j)\in E}\sum_{x_i,x_j\in\cX}\theta_{ij}(x_i,x_j)b_{ij}(x_i,x_j)+\sum_{i\in V}\sum_{x_i\in\cX}\theta_i(x_i)b_i(x_i)\, ,\label{eq:LP1}\\
  \mbox{subject to}&\;\;\; \sum_{x_j\in \cX} b_{ij}(x_i,x_j) = b_i(x_i)\;\;\;\;\;\;\;\;\;\;\;\;\;\;
                     \forall (i,j)\in E\, ,\label{eq:LP2}\\
                                        &\;\;\; b_i \in \Delta_{\cX} \;\;\; \forall i \in V, \;\;\; b_{ij}\in\Delta_{\cX\times \cX}\;\;\; \forall (i,j)\in E\ ,\label{eq:LP3}
\end{align}
where $\Delta_S$ denotes the simplex of probability distributions over set $S$. 
The decision variables are referred to as `beliefs',
and their feasible set 
is a relaxation of the polytope of marginals of distributions.
The beliefs  satisfy  the constraints on
marginals involving at most two variables 
connected by an edge.

Loopy belief propagation is successful on some applications,
e.g. sparse locally tree-like graphs that
arise, for instance, decoding modern error correcting codes
\cite{richardson2008modern} or in random constraint
satisfaction problems \cite{mezard2009information}.
However, in more structured instances -- arising for example in computer vision --
BP can be substantially improved by accounting for local dependencies within subsets
of more than two variables.
This is achieved by generalized belief propagation (GBP)
\cite{yedidia2005constructing} where
the decision variables are beliefs  $b_R$
that are defined on subsets of vertices (a `region')  $R \subseteq [n]$,
and that represent the marginal
distributions of the variables in that region.
The basic constraint on the beliefs is
the linear marginalization constraint:
$\sum_{\bx_{R\setminus S}} b_R(\bx_R) = b_S(\bx_S)$,
holding whenever $S\subseteq R$. Hence GBP itself is closely
related to LP relaxation of the polytope of marginals of probability distributions.
The relaxation becomes tighter as larger regions are incorporated.
In a prototypical application, $G$ is a two-dimensional grid, and
regions are squares induced by four contiguous vertices (plaquettes),
see Figure~\ref{fig:grid-sos}, left frame.
Alternatively in the right frame of the same figure,
the regions correspond to triangles.

The LP relaxations that correspond to GBP are closely related to
the Sherali-Adams hierarchy \cite{sherali1990hierarchy}.
Similar to GBP, the variables within this hierarchy are beliefs over
subsets of variables $\bb_R = (b_R(\bx_R))_{\bx_R\in\cX^R}$
which are consistent under
marginalization: $\sum_{\bx_{R\setminus S}} b_R(\bx_R) = b_S(\bx_S)$.
However, these two approaches differ in an important point: Sherali-Adams hierarchy uses 
beliefs over \emph{all subsets} of $|R|\le d$ variables,
where $d$ is the degree in the hierarchy;
this leads to an LP of size $\Theta(n^d)$. In contrast, GBP only retains
regions that are contiguous in $G$. If $G$ has maximum degree $k$,
this produces an LP of size $\O(nk^d)$,
a reduction which is significant for large-scale problems.

Given the broad empirical success of GBP,
it is natural to develop better methods for inference in graphical models
using tighter convex relaxations.
Within combinatorial optimization,
it is well understood that the semidefinite programming (SDP) relaxations 
provide superior approximation guarantees with respect to LP \cite{goemans1995improved}.
Nevertheless, SDP has found limited applications in inference tasks for
graphical models for at least two reasons.
A \emph{structural reason}: 
standard SDP relaxations (e.g. \cite{goemans1995improved})
do not account exactly for correlations
between neighboring vertices in the graph which is
essential for structured graphical models.
As a consequence, BP or GBP often outperforms basic SDPs.
A \emph{computational reason}: basic SDP relaxations involve
$\Theta(n^2)$ decision variables,
and generic interior point solvers do not scale well for the large-scale applications.
An exception is \cite{wainwright2004semidefinite} which employs the simplest SDP 
relaxation
(degree $2$ Sum-Of-Squares, see below) in conjunction with a relaxation of the entropy and interior point methods --
higher order relaxations are briefly discussed without implementation
as the resulting program suffers from the aforementioned limitations.

In this paper, we revisit MAP inference in graphical models via SDPs,
and propose an approach that carries over the favorable performance guarantees of SDPs into
inference tasks. For simplicity,
we focus on models with binary variables, but we believe 
that many of the ideas developed here can be naturally extended to other finite domains.
We present the following contributions:

\noindent{\bf Partial Sum-Of-Squares relaxations.}
  We use SDP hierarchies, specifically the Sum-Of-Squares (SOS)
  hierarchy \cite{shor1987class,lasserre2001explicit,parrilo2003semidefinite}
  to formulate tighter SDP relaxations for binary MAP inference 
  that account exactly for the joint distributions of small subsets 
  of variables $\bx_R$, for $R\subseteq V$.
  However, 
  SOS introduces decision variables for all subsets $R\subseteq V$
  with $|R|\le d/2$ ($d$ is a fixed even integer),
  and hence scales poorly for large-scale inference problems. 
  We propose a similar modification as in GBP.
  Instead of accounting for all subsets $R$ with $|R|\le d/2$,
  we only introduce decision variables to represent 
  a certain family of such subsets (regions) of vertices in $G$. The resulting SDP has
  (for $d$ and the maximum degree of $G$ bounded) only $\O(n^2)$ decision variables
  which is suitable for practical implementations.
  We refer to these relaxations as Partial Sum-Of-Squares (PSOS), cf. Section \ref{sec:PSOS}.

\noindent{\bf Theoretical analysis.}
  In Section \ref{sec:Theoretical}, we prove that suitable $\PSOS$ relaxations
  are tight for certain classes of graphs, including planar graphs, with $\btv=0$.
  While this falls short of explaining the empirical results
  (which uses simpler relaxations, and $\btv\neq 0$), it points in the right direction. 
 
\noindent{\bf Optimization algorithm and rounding.}
  Despite the simplification afforded by PSOS, interior-point solvers still scale poorly to
  large instances. In order to overcome this problem,
  we adopt a non-convex approach proposed by Burer and Monteiro \cite{burer2003nonlinear}.
  We constrain the rank of the SDP matrix in PSOS to be at most $r$, and 
  solve the resulting non-convex problem using a
  trust-region coordinate ascent method, cf. Section \ref{sec:opt-psos}.
  Further, we develop a
  rounding procedure called Confidence Lift and Project (CLAP)
  which iteratively uses PSOS relaxations to obtain an integer solution,
  cf. Section~\ref{sec:opt-clap}.
 
\noindent{\bf Numerical experiments.}
  In Section~\ref{sec:Numerical}, we present numerical experiments
  with PSOS by solving problems of size up to $10,000$ within several minutes.
  While additional work is required to scale
  this approach to massive sizes, we view this as an
  exciting proof-of-concept. To the best of our knowledge,
  no earlier attempt was successful in scaling higher order SOS relaxations 
  beyond tens of dimensions.
  More specifically, we carry out
  experiments with two-dimensional grids -- an image denoising problem,
  and Ising spin glasses.
  We demonstrate through extensive numerical studies that PSOS significantly
  outperforms BP and GBP in the inference tasks we consider.

\begin{figure}
  \centering
  \includegraphics[width=0.4\textwidth]{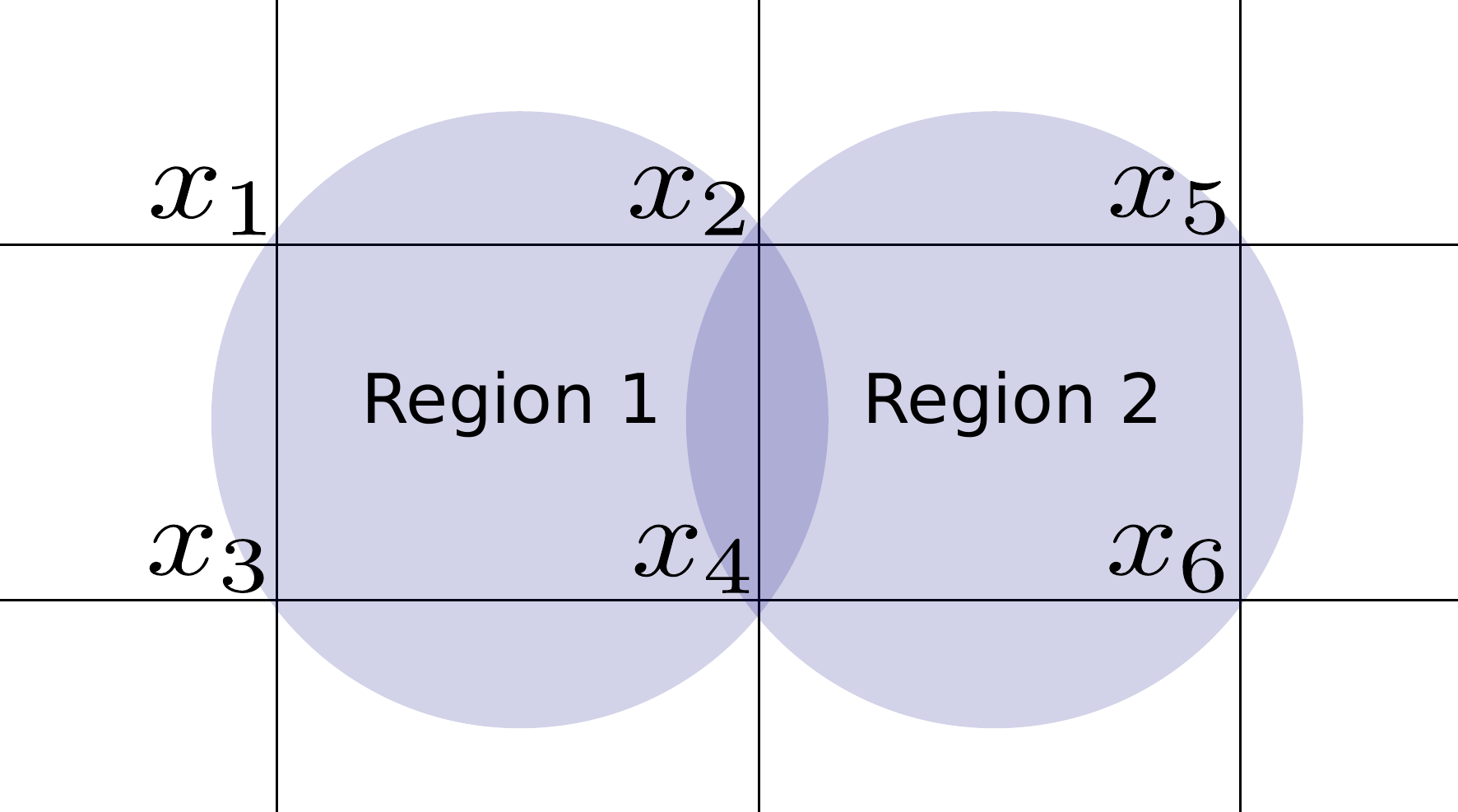}\hspace{0.5cm}
  \includegraphics[width=0.4\textwidth]{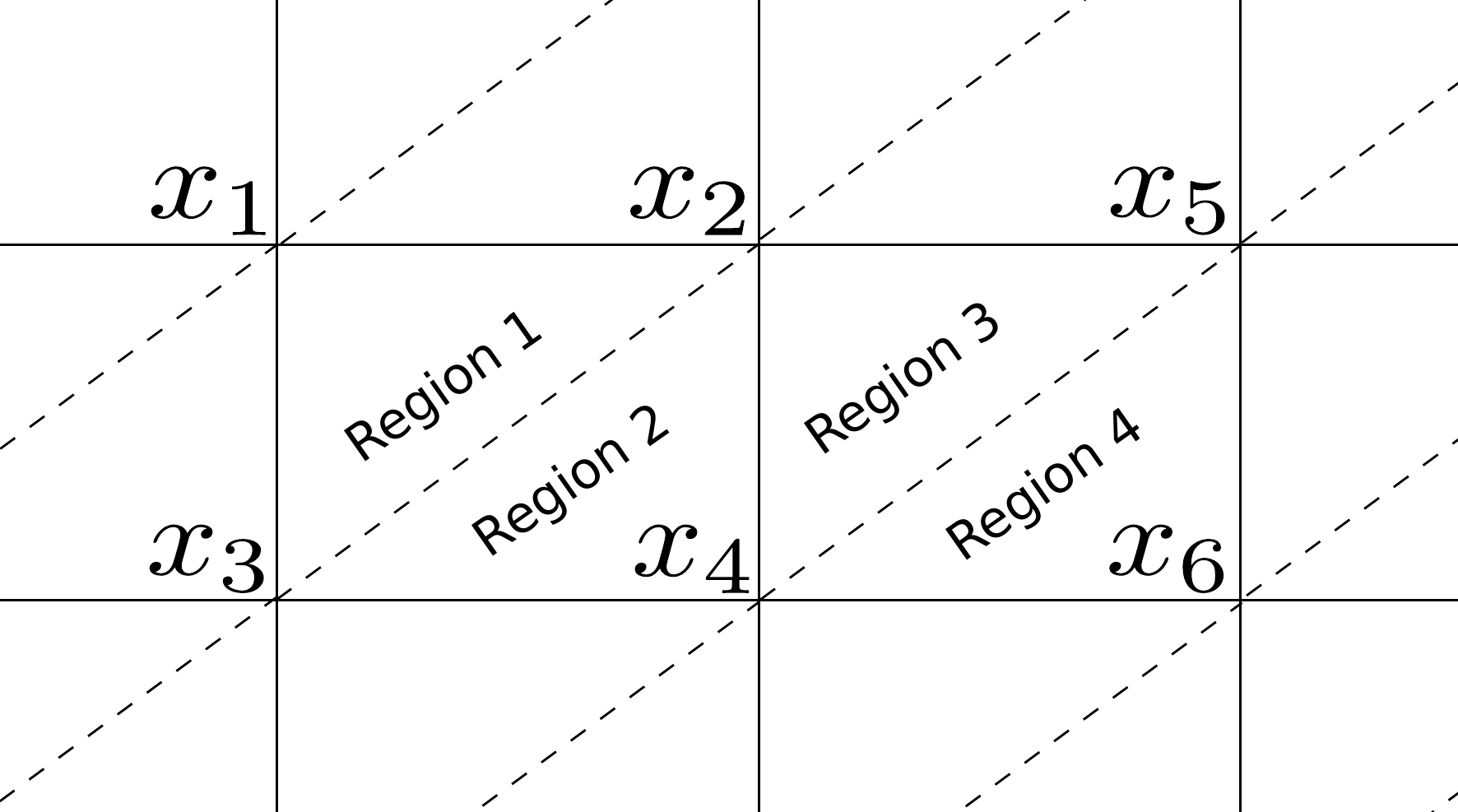}
  \caption{A two dimensional grid, and two typical choices for regions for GBP and PSOS.\\
    Left: Regions are plaquettes comprising four vertices.
    Right: Regions are triangles.}\label{fig:grid-sos}
  \vspace{-.2in}
\end{figure}

\vspace{-.1in}
\section{Partial Sum-Of-Squares Relaxations}
\vspace{-.1in}
\label{sec:PSOS}

For concreteness, throughout the paper we focus on
pairwise models with binary variables. 
We do not expect fundamental problems extending the same approach to other domains.
For binary variables $\bx = (x_1,x_2, ..., x_n)$, 
MAP estimation amounts to solving the following optimization problem
%
\begin{align}\tag{INT}
  \underset{\bx}{\mbox{maximize}}&\;\;\; \sum_{(i,j)\in E}\te_{ij}x_ix_j+\sum_{i\in V}\tv_ix_i\, ,\label{eq:Opt}\\
  \mbox{subject to}&\;\;\; x_i\in\{+1,-1\}\, ,\;\;\;\; \forall i\in V\, ,\nonumber
\end{align}
%
where $\bte = (\te_{ij})_{1\le i,j\le n}$ and $\btv = (\tv_i)_{1\le i\le n}$
are the parameters of the graphical model.

For the reader's convenience, we recall a few basic facts about SOS relaxations,
referring to \cite{barak2017proofs} for further details.
For an even integer $d$, $\SOS(d)$ is an SDP relaxation of \opt\ 
with decision variable
$X:\binom{[n]}{\le d}\to \reals$ where $\binom{[n]}{\le d}$ denotes the set of subsets $S\subseteq [n]$ of size $|S|\le d$; it is given as
\begin{align}\tag{SOS}
  \underset{X }{\mbox{maximize}}&\;\;\;\sum_{(i,j)\in E}\te_{ij}X(\{i,j\})+\sum_{i\in V}\tv_iX(\{i\})\, ,\label{eq:SOS}\\
  \mbox{subject to}&\;\;\; X(\emptyset) =1, \;\;\M(X)\mge 0\, .\nonumber
\end{align}
The moment matrix $\M(X)$ is indexed by sets $S, T\subseteq [n]$, $|S|, |T|\le d/2$,
and has entries  $\M(X)_{S,T}= X(S\triangle T)$ with $\triangle$ denoting
the symmetric difference of two sets. Note that $\M(X)_{S,S}= X(\emptyset) = 1$. 

We can equivalently represent $\M(X)$ as a Gram matrix by
letting $\M(X)_{S,T}= \<\bsigma_{S},\bsigma_T\>$ for a collection of
vectors $\bsigma_S\in\reals^r$ indexed by $S\in\binom{[n]}{\le d/2}$.
The case $r=\big|\binom{[n]}{\le d/2}\big|$ can represent any semidefinite matrix;
however, in what follows it is convenient from a computational perspective
to consider smaller choices of $r$.
The constraint $\M(X)_{S,S}=1$ is equivalent to $\|\bsigma_S\|=1$,
and the condition
$M(X)_{S,T} = X(S\triangle T)$ can be equivalently written as
\begin{align}
  \<\bsigma_{S_1},\bsigma_{T_1}\>= \<\bsigma_{S_2},\bsigma_{T_2}\>\,,\;\;\;\;\;
  \forall S_1\triangle T_1= S_2\triangle T_2.\label{eq:GeneralConstraints}
\end{align}
In the case $d=2$, SOS($2$) recovers the classical
Goemans-Williamson SDP relaxation \cite{goemans1995improved}.
%
%

\begin{figure}[t]
  \includegraphics[width=2.6in]{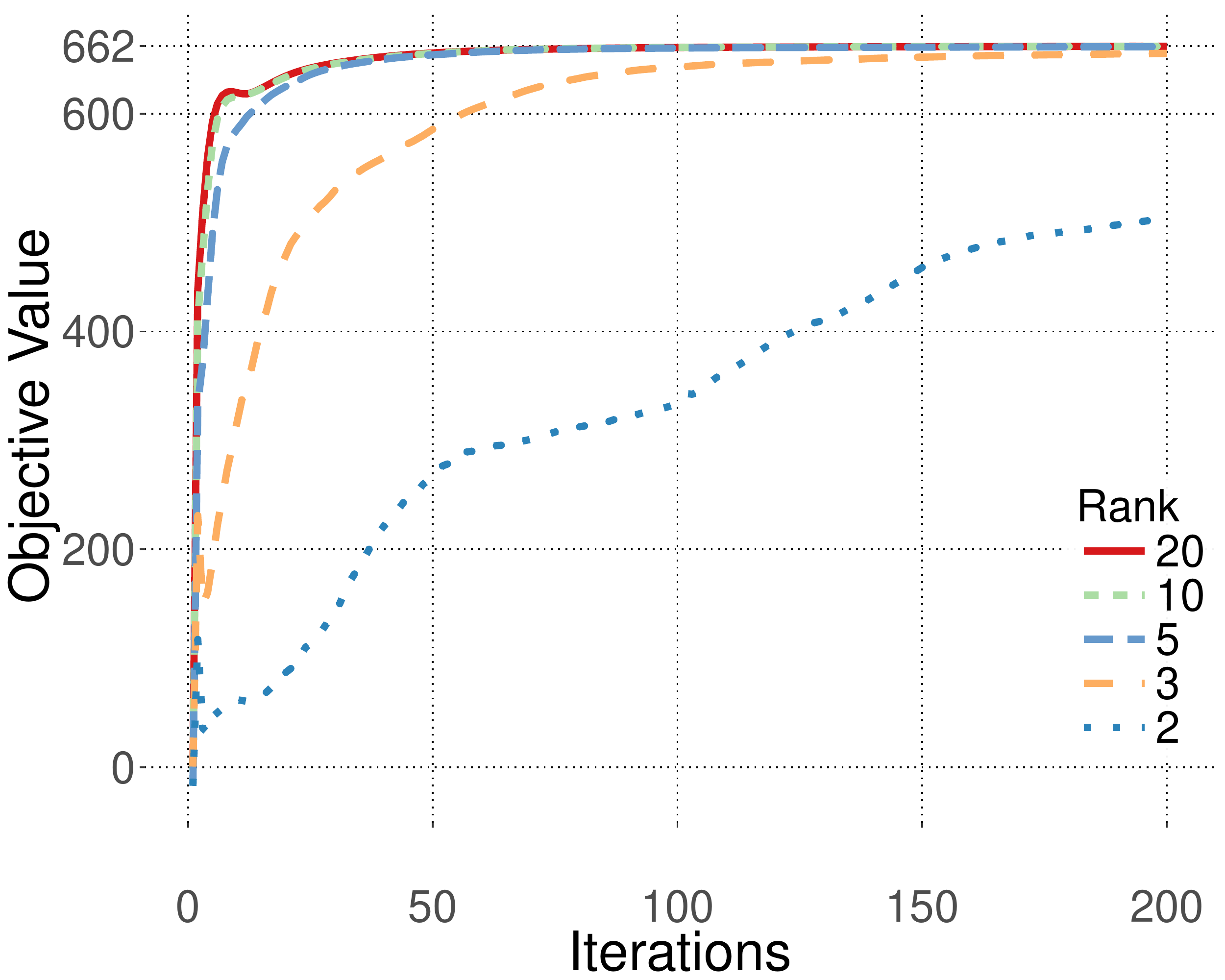}\hfill
  \includegraphics[width=2.6in]{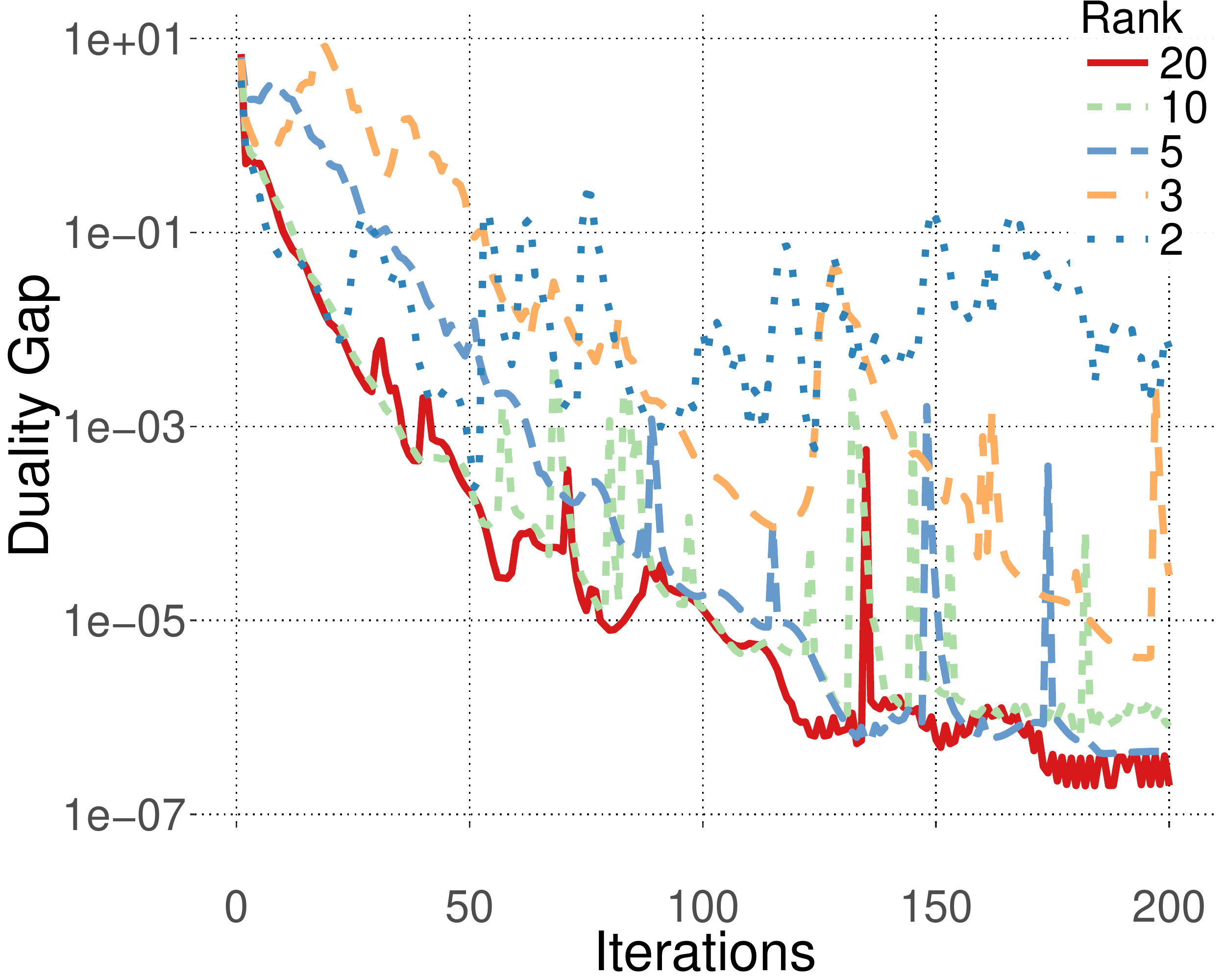}
  \caption{\label{fig:rank-effect}
    Effect of the rank constraint $r$ on $n=400$ square lattice ($20 \times 20$):
    Left plot shows the change in the value of objective at each iteration.
    Right plot shows the duality gap of the Lagrangian.}
  \vspace{-.15in}
\end{figure}
In the following, we consider the simplest higher-order SDP, namely SOS($4$) for which
the general constraints in Eq.~\eqref{eq:GeneralConstraints} can be listed explicitly.
Fixing a region $R\subseteq V$, and defining the Gram vectors
$\bsigma_{\emptyset}, (\bsigma_i)_{i\in V}, (\bsigma_{ij})_{\{i,j\}\subseteq  V}$,
we list the constraints that involve vectors 
$\bsigma_{S}$ for $S\subseteq R$ and $|S|=1, 2$:
\begin{alignat*}{5}
  & \norm{\s_i}=1  \hspace{1.1in}&&\forall i \in  S\cup \{\emptyset\}  \tag{Sphere $\spherec$\ \ },\\
  &\<\s_i, \s_j\> = \<\s_{ij}, \s_\emptyset\>  &&\forall i,j \in S, \tag{Undirected $i - j$} \\
  &\<\s_i, \s_{ij}\> = \<\s_{j}, \s_\emptyset\>  &&\forall i,j \in S,  \tag{Directed $i \rightarrow j$}\\
  &\<\s_{i}, \s_{jk}\> = \<\s_{k}, \s_{ij}\> &&\forall i,j,k \in S, \tag{V-shaped $^{i}_{j}V^k$}\\[-.5em]
  &\<\s_{ij}, \s_{jk}\> = \<\s_{ik}, \s_\emptyset\>  &&\forall i,j,k \in S, \tag{Triangle $\overset{\raisebox{-2pt}{\ $i$}}{_{j}\triangle_k}$}\\
  &\<\s_{ij}, \s_{kl}\> = \<\s_{ik}, \s_{jl}\>  &&\forall i,j,k,l \in   S. \tag{Loop $^{i}_{k}\Box^j_l$}
\end{alignat*}
Given an assignment of the Gram vectors
$\bsigma= (\bsigma_{\emptyset},(\bsigma_i)_{i\in V}, (\bsigma_{ij})_{\{i,j\}\subseteq V})$,
we denote by 
$\bsigma|_R$ its restriction to $R$, namely
$\bsigma|_R=(\bsigma_{\emptyset},(\bsigma_i)_{i\in R}, (\bsigma_{ij})_{\{i,j\}\subseteq R})$.
We denote by $\bOmega(R)$, the set of vectors $\bsigma|_R$
that satisfy the above constraints.
With these notations, the SOS($4$) SDP can be written as
\begin{align}\tag{SOS($4$)}
  \underset{\bsigma}{\mbox{maximize}}&\;\;\; \sum_{(i,j)\in E}\te_{ij}\<\bsigma_i,\bsigma_j\>+\sum_{i\in V}\tv_i\<\bsigma_i,\bsigma_{\emptyset}\>\, ,\label{eq:SOS4}\\
  \mbox{subject to}&\;\;\; \bsigma\in \bOmega(V)\, .\nonumber
\end{align}
A specific Partial SOS (PSOS) relaxation is defined by a collection of regions $\cR =\{R_1,R_2,\dots, R_m\}$, $R_i\subseteq V$. We will require $\cR$ to be a covering,
i.e.  $\cup_{i=1}^m R_i = V$ and for each $(i,j)\in E$ there exists $\ell\in[m]$
such that $\{i,j\}\subseteq R_{\ell}$.
Given such a covering, the PSOS($4$) relaxation is
\begin{empheq}[box=\fbox]{align}
  \tag{PSOS($4$)}
  \underset{\bsigma}{\mbox{maximize}}&\;\;\; \sum_{(i,j)\in E}\te_{ij}\<\bsigma_i,\bsigma_j\>+\sum_{i\in V}\tv_i\<\bsigma_i,\bsigma_{\emptyset}\>\, ,\label{eq:PSOS4}\\
  \mbox{subject to}&\;\;\; \bsigma|_{R_i}\in \bOmega(R_i)\;\;\;\; \forall i\in\{1,2,\dots,m\}\, .\nonumber
\end{empheq}
Notice that variables $\bsigma_{ij}$ only enter the above program
if $\{i,j\}\subseteq R_{\ell}$ for some $\ell$.
As a consequence, the dimension of the above optimization problem
is $\O(r\sum_{\ell=1}^m|R_\ell|^2)$, which is
$\O(nr)$ if the regions have bounded size; this will be the case in our implementation. 
Of course, the specific choice of regions $\cR$ is crucial for
the quality of this relaxation. A natural heuristic is to
choose each region $R_{\ell}$ to be a subset of contiguous vertices in $G$,
which is generally the case for GBP algorithms.

\vspace{-.15in}
\subsection{Tightness guarantees}
\vspace{-.1in}
\label{sec:Theoretical}

Solving exactly \opt\ is NP-hard even if $G$ is a three-dimensional grid
\cite{barahona1982computational}. Therefore, 
we do not expect PSOS($4$) to be tight for general graphs $G$. 
On the other hand,
in our experiments (cf. Section \ref{sec:Numerical}), 
PSOS($4$) systematically achieves the exact maximum of
\opt\ for two-dimensional grids with random edge and vertex parameters
$(\te_{ij})_{(i,j)\in E}$, $(\tv_i)_{i\in V}$.
This finding is quite surprising and calls for a theoretical explanation. 
While full understanding remains an open problem,
we present here partial results in that direction.

Recall that a cycle in $G$ is a sequence of distinct vertices $(i_1,\dots,i_{\ell})$ such that, for each $j\in[\ell]\equiv \{1,2,\dots,\ell\}$, $(i_j,i_{j+1}) \in E$
(where $\ell+1$ is identified with $1$). The cycle is chordless if there is no $j,k\in[\ell]$, with $j-k\neq \pm 1$ $\mod \ell$ such that $(i_j,i_k)\in E$.
  We say that a collection of regions $\cR$ on graph $G$ is \emph{circular} if for each chordless cycle in $G$
  there exists a region in $R \in \cR$ such that  all vertices of the cycle belong to $R$.
We also need the following straightforward notion of contractibility. A \emph{contraction} of $G$ is a new graph obtained
by identifying two vertices connected by an edge in $G$. 
$G$ is \emph{contractible} to $H$ if there exists a sequence of contractions 
transforming $G$ into $H$.

The following theorem  is a direct consequence of a result of Barahona and Mahjoub \cite{barahona1986cut} (see Supplement for a proof).
\begin{theorem}\label{thm::psos-tight}
  Consider the problem {\rm \opt} with $\btv=0$.
  If $G$ is not contractible to $K_5$  (the complete graph over $5$ vertices),
  then {\rm PSOS($4$)} with a circular covering $\cR$ is tight.
\end{theorem}
\vspace{-0.08in}
 The assumption that $\btv = 0$ can be made without loss
of generality (see Supplement for the reduction from the general case). 
Furthermore, \opt\ can be solved in polynomial time if
$G$ is planar, and $\btv=0$ \cite{barahona1982computational}. 
Note however, the reduction from $\btv\neq 0$ to $\btv=0$ can
transform a planar graph to a non-planar graph. 
This theorem implies that (full) SOS($4$) is also tight if $G$
is not contractible to $K_5$. Notice that planar graphs are not contractible
to $K_5$, and we recover 
the fact that \opt\ can be solved in polynomial time if $\btv=0$.
This result falls short of explaining the empirical findings in Section \ref{sec:Numerical}, 
for at least two reasons. Firstly the reduction to $\btv = 0$ 
induces $K_5$ subhomomorphisms for grids.
Second, the collection of regions $\cR$ described in the
previous section does not include all chordless cycles. Theoretically 
understanding the empirical
performance of PSOS(4) as stated remains open. However, 
similar cycle constraints have proved useful in analyzing LP relaxations \cite{weller2016tightness}. 

\vspace{-.12in}
\section{Optimization Algorithm and Rounding}
\label{sec:opt}
\vspace{-.08in}
\subsection{Solving PSOS($4$) via Trust-Region Coordinate Ascent}
\vspace{-0.05in}
\label{sec:opt-psos}
\setlength{\textfloatsep}{5pt}

\renewcommand{\sfdefault}{cmss}

\begin{algorithm}[t]
  \DontPrintSemicolon
  \SetKwFunction{parsos}{Partial-SOS}
  \SetKwData{Reliables}{Reliables}
  \SetKwData{Actives}{Actives}
  \SetKwInOut{Input}{Input}\SetKwInOut{Output}{Input}
  
   \Input{$G=(V,E)$, $\bte\in \reals^{n \times n}$, $\btv \in \reals^n$,
    $\bsigma \in \reals^{r \times (1+|V|+|E|)}$,
    $\Reliables=\emptyset$ }
  
  {
     $\Actives = V \cup E \setminus \Reliables$, and $\Delta\! = \!1$,
    
     \While{$\Delta >$ \tol}{
       $\Delta = 0$
      
       \For{$s\in \Actives $}{\label{algst::choose-var}
         \uIf(\tcc*[f]{$s\in V$ is a vertex}){$s\in V$}{
          \label{algst::linear-term-v}
          $\bc_s = \sum_{t \in \partial s}\te_{st} \bsigma_t + \tv_s\bsigma_{\emptyset}$
        }
         \uElse(\tcc*[f]{$s = (s_1, s_2) \in E$ is an edge}){
          \label{algst::linear-term-e}
          $\bc_s = \te_{s_1s_2} \s_\emptyset + \tv_{s_1}\s_{s_2} + \tv_{s_2}\s_{s_1}$
        }

        \vspace{.05in}
        
        Form matrix $\bA_s$, vector $\bb_s$,
        and the corresponding Lagrange multipliers $\blambda_s$ (see text).
        
          \vspace{.05in}
        
        $\s^{\rm new}_s \longleftarrow \underset{\norm{\s}=1}{\arg\max}\left\{ \<\bc_s, \s\>+
          \frac{\rho}{2}\norm{\bA_s\s - \bb_s + \blambda_s}^2\right\}$\tcc*[f]{sub-problem} \label{algst::trust-region-step}

        $\Delta \longleftarrow \Delta + \norm{\s^{\rm new}_s - \s_s}^2 + \norm{\bA_s\s_s - \bb_s}^2$
        
         $\s_s\longleftarrow \s^{\rm new}_s$ \tcc*[f]{update variables}
        
         $\blambda_s \longleftarrow \blambda_s + \bA_s\s_s - \bb_s$
        
      }
    }
  }
  \cprotect\caption{ \parsos \label{alg::parsos}\vspace{-0.1in}}
\end{algorithm}
\renewcommand{\rmdefault}{ptm}
\renewcommand{\sfdefault}{phv}
We will approximately solve PSOS($4$) while keeping $r=\O(1)$.
Earlier work implies that (under suitable genericity 
condition on the SDP) there exists an optimal solution with rank $\sqrt{2 \text{ \# constraints}}$ \cite{pataki1998rank}. 
Recent work \cite{boumal2016non}
shows that for $r>\sqrt{2 \text{ \# constraints}}$, the non-convex optimization problem has 
no non-global local maxima.
For SOS($2$), \cite{mei2017solving} proves that 
setting $r=\O(1)$ is sufficient for achieving $\O(1/r)$ relative error from
the global maximum for specific choices of potentials $\bte,\btv$. We find that there is
little or no improvement beyond $r=10$ (cf. Figure \ref{fig:rank-effect}).

We will assume that $\cR = (R_1,\dots,R_m)$ is a covering of $G$ (in the sense introduced in the previous section), and --without loss of generality--
we will assume that the edge set is
\begin{align}
  E= \big\{(i,j)\in V\times V:\;\;\; \exists \ell\in [m] \;\;\;\mbox{such that} \;\;\;\{i,j\}\subseteq R_{\ell}\big\} \, .\label{eq:ConditionCov}
\end{align}
In other words, $E$ is the maximal set of edges that is compatible with $\cR$ being a covering. This can always  be achieved by adding new edges 
$(i,j)$ to the original edge set with $\te_{ij} =0$. 
Hence, the decision variables $\bsigma_s$ are indexed by
$s \in \cS = \{\emptyset\}\cup V\cup E$.
Apart from the norm constraints, all other consistency constraints
take the form $\<\bsigma_{s},\bsigma_{r}\> = \<\bsigma_{t},\bsigma_{p}\>$
for some 4-tuple of indices $(s,r,t,p)$. We denote the set of all such 4-tuples
by $\cC$, and construct the augmented Lagrangian of PSOS($4$) as
\begin{align*}
  \cL(\bsigma,\blambda) = &\sum_{i\in V}\tv_i\<\bsigma_i,\bsigma_{\emptyset}\>
                            +\!\!\sum_{(i,j)\in E}\te_{ij}\<\bsigma_i,\bsigma_j\>
                            +\frac{\rho}{2}\!\!
                            \sum_{(s,r,t,p)\in \cC} \Big(\<\bsigma_{s},\bsigma_{r}\> - \<\bsigma_{t},\bsigma_{p}\>+\lambda_{s,r,t,p}\Big)^2\, . 
\end{align*}
At each step, our algorithm execute two operations:
$(i)$ maximize the cost function with respect to one of the vectors
$\bsigma_s$; $(ii)$ perform one step
of gradient descent with respect to the corresponding subset of Lagrangian parameters,
to be denoted by $\blambda_s$.
%
More precisely, fixing $s\in \cS\setminus \{\emptyset\}$ (by rotational invariance,
it is not necessary to update $\bsigma_{\emptyset}$), we note
that  $\bsigma_s$ appears in the constraints linearly (or it does not appear).
Hence, we can write these constraints in the form $\bA_s\bsigma_s = \bb_s$
where $\bA_s, \bb_s$ depend on $(\bsigma_r)_{r\neq s}$ but not on $\bsigma_s$.
We stack the corresponding Lagrangian parameters in a vector $\blambda_s$;
therefore the Lagrangian term involving $\bsigma_s$ reads
$(\rho/2) \|\bA_s\bsigma_s-\bb_s+\blambda_s\|^2$. On the other hand, 
the graphical model contribution is that the first two terms in
$\cL(\bsigma,\blambda)$ are linear in
$\bsigma_s$, and hence they can be written as $\<\bc_s,\bsigma_s\>$. Summarizing, we have
\begin{align}
  \cL(\bsigma,\blambda) = &\<\bc_s,\bsigma_s\>+\|\bA_s\bsigma_s-\bb_s+\blambda_s\|^2+
                            \widetilde{\cL}\big((\bsigma_r)_{r\neq s},\blambda\big) \, .\label{eq:BlockLagrangian}
\end{align}
It is straightforward to compute $\bA_s$, $\bb_s, \bc_s$; in particular, 
for $(s,r,t,p)\in \cC$, the rows of $\bA_s$ and $\bb_s$ are indexed by $r$
such that the vectors
$\bsigma_{r}$ form the rows of $\bA_s$, and
$\<\bsigma_{t},\bsigma_{p}\>$ form the corresponding entry of $\bb_s$.
Further, if $s$ is a vertex and $\partial s$ are its neighbors, 
we set $\bc_s = \sum_{t\in \partial s}\bte_{st} \s_t + \btv_s\s_{\emptyset}$ 
while if $s=(s_1,s_2)$ is an edge, we set $\bc_{s}= \bte_{s_1s_2} \s_\emptyset + \btv_{s_1}\s_{s_2}+\btv_{s_2}\s_{s_1}$. 
Note that we are using the equivalent representations
$\<\bsigma_i,\bsigma_j\>=\<\bsigma_{ij},\bsigma_{\emptyset}\>$, $\<\bsigma_{ij},\bsigma_j\>=\<\bsigma_{i},\bsigma_{\emptyset}\>$, and
$\<\bsigma_{ij},\bsigma_i\>=\<\bsigma_{j},\bsigma_{\emptyset}\>$.

Finally, we maximize Eq.~\eqref{eq:BlockLagrangian} with
respect to $\bsigma_s$ by a Mor{\'e}-Sorenson style method \cite{more1983computing}
(see for example \cite{erdogdu2015convergence} for potential improvements involving subsampling techniques).

\vspace{-.1in}
\subsection{Rounding via Confidence Lift and Project}
\vspace{-.1in}
\label{sec:opt-clap}

\renewcommand{\sfdefault}{cmss}

\begin{algorithm}[t]
  \DontPrintSemicolon
  \SetKwFunction{parsos}{Partial-SOS}
  \SetKwData{Reliables}{Reliables}
  \SetKwData{Promotions}{Promotions}
  \SetKwData{Confidence}{Confidence}
  \SetKwData{Actives}{Actives}
  \SetKwInOut{Input}{Input}%
  
   \Input{$G=(V,E)$, $\bte \in \reals^{n \times n}$, $\btv \in \reals^n$,
    regions $\cR=\{R_1,...,R_m\}$}

  {\vspace{.03in}
    
    Initialize variable matrix $\bsigma \in \reals^{r \times (1+|V|+|E|)}$ and set \Reliables $= \emptyset$.
    
      \While{ \Reliables\  $ \neq V \cup E$}{

        Run \parsos on inputs $G = (V,E)$, $\bte$, $\btv$, $\bsigma$, $\Reliables$
      \tcc*[f]{lift procedure} \label{algst::lift}
      
        \Promotions = $\emptyset$ and $\Confidence = 0.9$
      
        \While{$\Confidence > 0$ {\bf and} \Promotions $\neq\emptyset$}{
        \label{algst::proj}

          \For(\tcc*[f]{find promotions}){$s \in V \cup E \setminus \Reliables$}{

            \uIf{$|\<\s_\emptyset, \s_{s}\>| >\Confidence $ }{
             $\s_{s} = \text{sign}(\<\s_\emptyset, \s_{s}\>) \cdot \s_\emptyset$
            \tcc*[f]{project procedure}
            
             \Promotions $\longleftarrow$ \Promotions $\cup\  \{s_c\}$
          }
        }
          \uIf(\tcc*[f]{decrease confidence level}){\Promotions $=\emptyset$}{ 
           $\Confidence \longleftarrow \Confidence - 0.1 $
        }
          \Reliables $\longleftarrow$ \Reliables $\cup$ \Promotions
        \tcc*[f]{update \Reliables}
      }      
    }
  }
  \SetKwInOut{Output}{Output}
    \Output{$(\<\bsigma_i, \bsigma_\emptyset\>)_{i\in V} \in \{-1, +1 \}^{n}$}
    \cprotect\caption{\verb|CLAP: Confidence Lift And Project|  \label{alg::lift-proj}}
\end{algorithm}
\renewcommand{\rmdefault}{ptm}
\renewcommand{\sfdefault}{phv}
After Algorithm \ref{alg::parsos} generates an approximate optimizer $\bsigma$ for PSOS($4$),
we reduce its rank to produce a solution of the original combinatorial
optimization problem \opt. To this end, we interpret $\<\bsigma_i,\bsigma_{\emptyset}\>$
as our belief about the value of $x_i$ in  
the optimal solution of \opt, and $\<\bsigma_{ij},\bsigma_{\emptyset}\>$
as our belief about the value of $x_ix_j$.
This intuition can be formalized using the notion of
pseudo-probability \cite{barak2017proofs}. We then recursively 
round the variables about which we have strong beliefs;
we fix rounded variables in the next iteration, and
solve the induced PSOS($4$) on the remaining ones.

\vspace{-.03in}
More precisely, we set a confidence threshold $\Confidence$.
For any variable $\bsigma_s$ such that
$|\<\bsigma_s,\bsigma_{\emptyset}\>|>\Confidence$, we let $x_s = \sign(\<\bsigma_s,\bsigma_{\emptyset}\>)$ and
fix $\bsigma_s = x_s\, \bsigma_{\emptyset}$.
These variables $\bsigma_s$ are no longer updated, and instead 
the reduced SDP is solved. If no variable satisfies the confidence condition, the threshold is reduced until variables are found that satisfy it.
After the first iteration,
most variables yield strong beliefs and are fixed;
hence the consequent iterations have fewer variables and are faster.
\vspace{-.1in}
\section{Numerical Experiments}
\label{sec:Numerical}
\vspace{-.1in}
In this section, we validate the performance of the Partial SOS relaxation and
the CLAP rounding scheme on models defined on two-dimensional grids.
Grid-like graphical models are common in a variety of fields
such as computer vision \cite{sun2003stereo},
and statistical physics \cite{mezard2009information}.
In Section~\ref{sec:Denoising}, we study an image denoising example
and in Section~\ref{sec:SpinGlass}
we consider the Ising spin glass -- a model in statistical mechanics
that has been used as a benchmark for inference in graphical models.

\vspace{-.03in}
Our main objective is to demonstrate that Partial SOS can be used successfully
on large-scale graphical models, and is competitive with the following popular
inference methods: 
\vspace{-.08in}
\begin{itemize}[noitemsep, leftmargin=9pt]
\item \textbf{Belief Propagation - Sum Product (BP-SP)}:
  Pearl's belief propagation computes exact
  marginal distributions on trees \cite{pearl1986fusion}.
  Given a graph structured objective function $U(\bx)$,
  we apply BP-SP to the Gibbs-Boltzmann distribution
  $p(\bx) = \exp\{U(\bx)\}/Z$ using the standard sum-product update rules
  with an inertia of $0.5$ to help convergence \cite{yedidia2005constructing}, 
  and threshold the marginals at $0.5$.
\item \textbf{Belief Propagation - Max Product (BP-MP)}:
  By replacing the marginal probabilities in the sum-product updates with max-marginals,
  we obtain BP-MP, which can be used for exact inference on trees 
  \cite{mezard2009information}. For general graphs,
  BP-MP is closely related to an LP relaxation of the combinatorial problem \opt\ 
  \cite{yedidia2005constructing,weiss2001optimality}. 
   Similar to BP-SP, we use an inertia of 0.5. 
   Note that the Max-Product updates can be equivalently written as
   Min-Sum updates \cite{mezard2009information}.
\item \textbf{Generalized Belief Propagation (GBP)}:
  The decision variables in GBP are beliefs (joint probability distributions)
  over larger subsets of variables in the graph $G$, and they
  are updated in a message passing fashion
  \cite{yedidia2000generalized,yedidia2005constructing}.
  We use plaquettes in the grid (contiguous groups of
  four vertices) as the largest regions,
  and apply message passing with inertia $0.1$ 
  \cite{weiss2001optimality}.
\item \textbf{Partial SOS - Degree 2 (PSOS($2$))}:
  By defining regions as single vertices
  and enforcing only the sphere constraints,
  we recover the classical Goemans-Williamson SDP relaxation \cite{goemans1995improved}.
  Non-convex Burer-Monteiro approach 
  is extremely efficient in this case \cite{burer2003nonlinear}.
  We round the SDP solution by $\hat{x}_i = \sign(\<\s_i,\s_\emptyset\>)$
  which is closely related to the classical approach of \cite{goemans1995improved}.

\item \textbf{Partial SOS - Degree 4 (PSOS($4$))}:
  This is the algorithm developed in the present paper.
  We take the regions $R_{\ell}$ to be triangles,
  cf. Figure~\ref{fig:grid-sos}, right frame.
  In an $\sqrt{n}\times \sqrt{n}$ grid, we have $2(\sqrt{n} - 1)^2$ such regions
  resulting in $\O(n)$ constraints.
  In Figures~\ref{fig::monalisa} and \ref{fig::box-plot},
  PSOS($4$) refers to the CLAP rounding scheme applied together with
  PSOS($4$) in the lift procedure.
\end{itemize}

\vspace{-.15in}
\subsection{Image Denoising via Markov Random Fields}
\vspace{-.1in}
\label{sec:Denoising}

\begin{figure}[t]
  \centering
  \includegraphics[width=\linewidth]{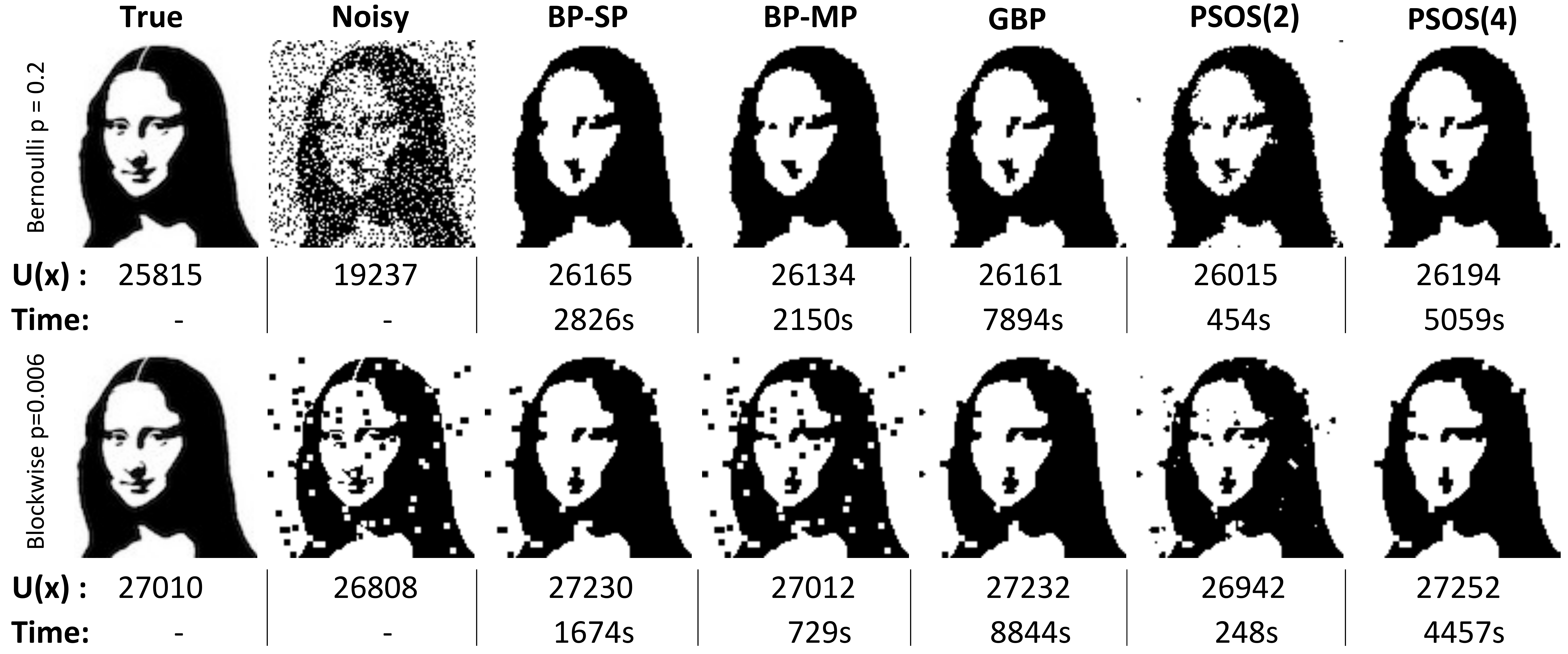}
  \caption{Denoising a binary image by maximizing
    the objective function Eq.~(\ref{eq:ObjDenoise}). 
    Top row: i.i.d. Bernoulli error with flip probability $p=0.2$ with $\theta_0=1.26$.
    Bottom row: blockwise noise where each pixel is the center of a
    $3\times 3$ error block independently with
    probability $p=0.006$ and $\theta_0=1$.\label{fig::monalisa}}
  \vspace{-0in}
\end{figure}

Given a $\sqrt{n}\times \sqrt{n}$ binary image $\bx_0\in\{+1,-1\}^n$, we generate a corrupted
version of the same image $\by\in\{+1,-1\}^n$. We then try to denoise $\by$ by maximizing the following 
objective function:
\vspace{-.03in}
\begin{align}
  U(\bx) = \sum_{(i,j)\in E} x_ix_j+\theta_0\sum_{i\in V} y_ix_i\, , \label{eq:ObjDenoise}\vspace{-.1in}
\end{align}
where the graph $G$ is the $\sqrt{n}\times \sqrt{n}$ grid, i.e.,
$V=\{i=(i_1,i_2):\;\; i_1,i_2\in\{1,\dots,\sqrt{n}\}\}$
and $E = \{(i,j):\;\; \|i-j\|_1=1\}$.
In applying Algorithm \ref{alg::parsos},
we add diagonals to the grid (see right plot in Figure~\ref{fig:grid-sos})
in order to satisfy the condition 
(\ref{eq:ConditionCov}) with corresponding weight $\te_{ij}=0$.

In Figure~\ref{fig::monalisa}, we report the output of various algorithms for a $100\times 100$ binary image.  
We are not aware of any earlier implementation of SOS($4$) beyond tens of variables,
while PSOS($4$) is applied here to $n=10,000$ variables.
Running times for CLAP rounding scheme (which requires several runs of PSOS($4$)) are of order an hour,
and are reported in Figure~\ref{fig::monalisa}.
We consider two noise models: i.i.d. Bernoulli noise and blockwise noise.
The model parameter $\theta_0$ 
is chosen in each case as to approximately optimize
the performances under BP denoising.
In these (as well as in 4 other experiments of the same type reported in the supplement),
PSOS($4$) gives consistently the best reconstruction 
(often tied with GBP), in reasonable time.
Also, it consistently achieves the largest value of
the objective function among all algorithms.

\vspace{-.07in}
\subsection{Ising Spin Glass}
\label{sec:SpinGlass}
\vspace{-.05in}
\begin{figure}[t]
  \centering
  \includegraphics[width=5.5in]{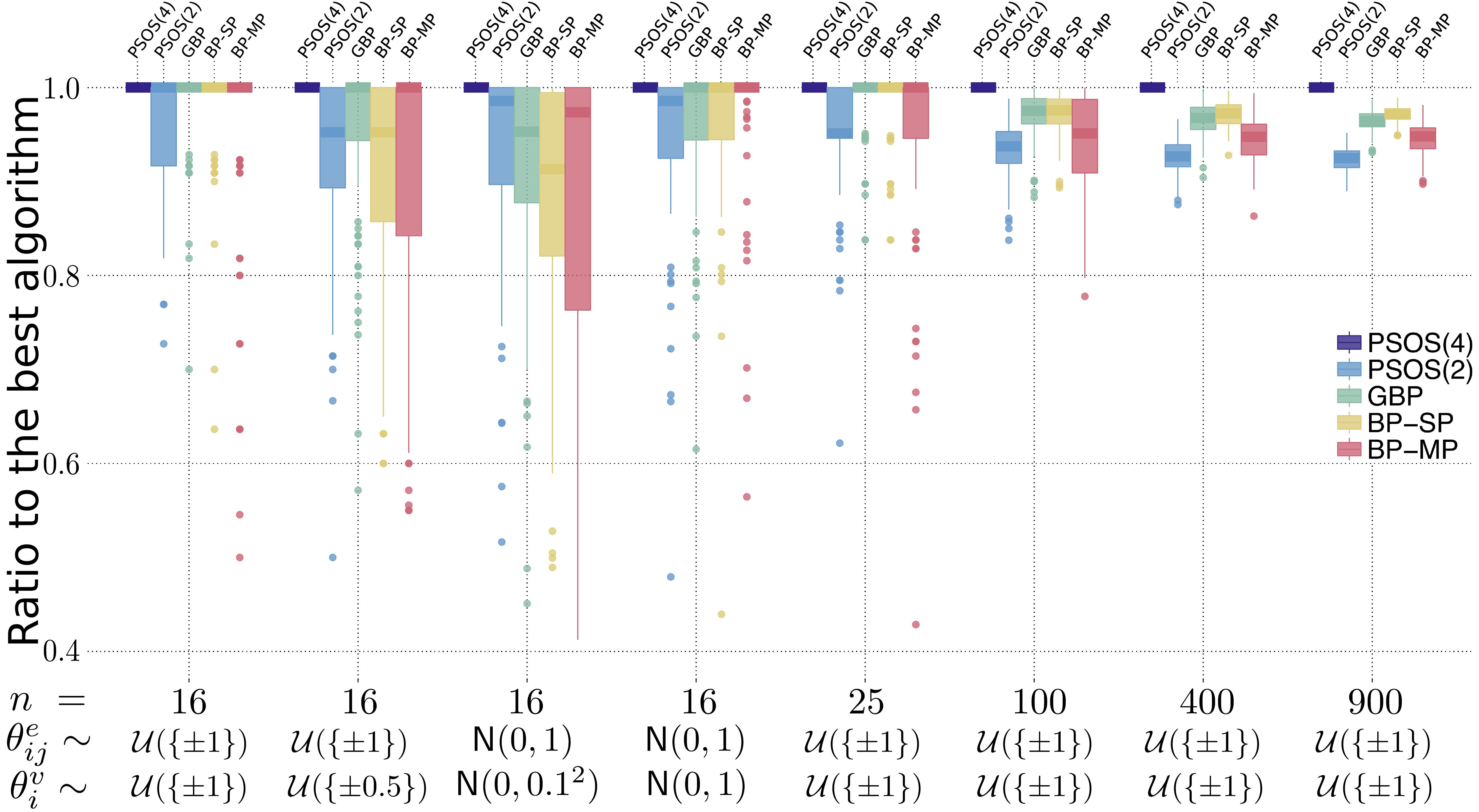}
  \caption{Solving the MAP inference problem \opt\ 
    for Ising spin glasses on two-dimensional grids.
    $\mathcal{U}$ and $\normal$ represent uniform and normal distributions.
    Each bar contains 100 independent realizations.
    We plot the ratio between the objective value achieved by that 
    algorithm and the exact optimum for $n\in\{16,25\}$,
    or the best value achieved by any of the 5 algorithms
    for $n\in \{100,400, 900\}$.\label{fig::box-plot}}
  \vspace{-.0in}
\end{figure}

The Ising spin glass (also known as Edwards-Anderson model \cite{edwards1975theory})
is one of the most studied models in statistical physics. 
It is given by an objective function of the form \opt\  with $G$ a $d$-dimensional
grid, and i.i.d. parameters  $\{\te_{ij}\}_{(i,j)\in E}$, $\{\tv_i\}_{i\in V}$.  
Following earlier work \cite{yedidia2005constructing},
we use Ising spin glasses as a testing ground for our algorithm.
Denoting the uniform and normal distributions by $\mathcal{U}$ and $\normal$ respectively,
we consider two-dimensional grids (i.e. $d=2$), and the following 
parameter distributions:
$(i)$ $\te_{ij} \sim \mathcal{U}(\{+1,-1\})$ and  $\tv_i\sim\mathcal{U}(\{+1,-1\})$,
$(ii)$ $\te_{ij} \sim \mathcal{U}(\{+1,-1\})$ and $\tv_i \sim \mathcal{U}(\{+1/2,-1/2\})$,
$(iii)$ $\te_{ij} \sim \normal(0, 1)$ and  $\tv_i \sim \normal(0, \sigma^2)$ with $\sigma=0.1$
(this is the setting considered in \cite{yedidia2005constructing}), and
$(iv)$ $\te_{ij} \sim \normal(0, 1)$ and  $\tv_i \sim \normal(0, \sigma^2)$ with $\sigma=1$.
For each of these settings, we considered grids of size $n\in \{ 16, 25, 100, 400, 900\}$.

In Figure~\ref{fig::box-plot}, we report the results of 8 experiments as a box plot.
We ran the five inference algorithms described above
on 100 realizations;
a total of 800 experiments are reported in Figure~\ref{fig::box-plot}.
For each of the realizations, we record the ratio of
the achieved value of an algorithm to the exact maximum (for $n\in\{16,25\}$),
or to the best value achieved among these algorithms (for $n\in \{100,400, 900\}$).
This is because for lattices of size 16 and 25,
we are able to run an exhaustive search
to determine the true maximizer of the integer program.
Further details are reported in the supplement.

\emph{In every single instance of 800 experiments,
  \emph{PSOS($4$)} achieved the largest objective value, and
  whenever this could be verified by exhaustive search
  (i.e. for $n\in\{16,25\}$) it achieved an exact maximizer of the integer program.}

%

\setlength{\bibsep}{2pt plus 2.7ex}
\bibliographystyle{alpha}
{\small
  \bibliography{bib}
}

\appendix

\section{Proof of Theorem~\ref{thm::psos-tight}}
Given the graph $G=(V,E)$, and parameters $\bte, \btv$,
we can construct a new graph by adding the extra vertex $\emptyset$, together with edges $\{(i,\emptyset):\, i\in V\}$ connecting it to all previous vertices, and edge
parameters $\te_{i,\emptyset} = \tv_{i}$ (while setting to $0$ the vertex parameters).
Therefore, one can always eliminate the linear term and work with the quadratic form.

We define the \emph{cut polytope} as
\eq{
  \cutp &\deq \text{Conv}\!\left( \left\{xx^T : x_i^2 = 1 \ \forall i \in V\right\} \right),
}
which is a convex hull of $2^n$ rank-1 matrices.
Introducing the interaction variables $X_{ij} = x_ix_j$,
the original optimization problem can be written without the linear term as
\eq{
  &\underset{X\in \reals^{n \times n}}{\mbox{maximize}}\
  \sum_{(i,j) \in E} W_{ij}X_{ij}\\
  &\text{subject to: }\ X \in \cutp. \nonumber
}

For an edge $e = (i,j)$, denote by $X_e$ the entry $X_{ij}$, and
for an edge set $F \subset E$ let $X(F)$ be the summation of entries $X_{ij}$
for which $(i, j) \in F$, i.e. $X(F) = \sum_{e\in F}X_e$.
Further, define the \emph{metric polytope} as
\eq{\label{eq::cyclic-ineq}
  \metp &\deq \{M \in \syms^{n} : |M_{e}| \leq  1 \ \forall e \in E, \\ 
  &\hspace{.9in} M(F) - M(C\setminus F) \geq 2 - |C| \text{ for }
  F \subset C, |F| \text{ is odd}, C \text{ is a simple cycle} \}.\nonumber
}
The inequalities that define the metric polytope are called \emph{cyclic inequalities}.
We recall the following result of Barahona and Mahjoub.
\begin{theorem}[Barahona and Mahjoub \cite{barahona1986cut}]
  \label{thm::barahona}
  $G = (V, E)$ is not contractible to $\K_5$ if and only if
  the cut polytope $\cutp$ is defined by the metric polytope $\metp$.
\end{theorem}
The above result implies that the cut polytope is defined by the metric polytope,
if the underlying graph is not contractible to $K_5$. However, 
cyclic inequalities are not sufficient to describe $K_5$.

\begin{proof}[Proof of Theorem \ref{thm::psos-tight}]
  Define the symmetric matrix $M \in \reals^{n \times n}$ as
  $M_{ij} = M_{ji} = \inner{\s_i, \s_j}$ and $M_{ii} =1$ for $i,j \in [n]$.
  Clearly, $M$ is positive semidefinite.
  Since $\cR = \{R_1, R_2, ..., R_m\}$ is a covering of $G$,
  for each vertex $ i \in V$,
  there exists $j \in [m]$ such that $i \in R_j$.
  We have $\Sigma(R_j)$ satisfying degree-4 SOS
  constraints, which implies that relaxed variable $\s_i$ is
  on the unit sphere. Therefore, the entries of $M$ satisfy
  \eq{
    \abs{M_{ij}} =& \abs{\inner{\s_i, \s_j}} \leq \norm{\s_i} \norm{\s_j}, \\
    \leq& 1,\nonumber
  }
  by the Cauchy-Schwartz inequality.
  Similarly for
  an edge $(i,j)$, there exists $k \in [m]$ such that $i$
  and $j$ both belong to $R_k$. Therefore, the variable $\s_{ij}$ satisfies
  degree-4 SOS constraints, which in turn implies that it is
  on the unit sphere.

  Let $C = \{e_1, e_2, ..., e_N\}$ be a chordless cycle of length $N$
  such that $e_1$ and $e_N$ share a common vertex.
  There exists $p \in [m]$
  such that each node defining the elements of $C$ belongs to the region $R_p$.
  Assume that the nodes $i,j,k \in R_p$. Then,
  \eq{
    M_{ij} = \inner{\s_{i}, \s_j}=\inner{\s_{ij}, \s_0},
  }
  by the undirected constraints. Moreover, by using
  the triangle constraints we can write
  \eq{
    0 \leq& \frac{1}{4}\norm{\s_{ij} + \s_{jk} - \s_{ik} - \s_0}^2, \\
    = & 1 + \frac{1}{2} \inner{\s_{ij}, \s_{jk}}
    - \frac{1}{2} \inner{\s_{ij}, \s_{ik}}\nonumber
    - \frac{1}{2} \inner{\s_{ij}, \s_{0}}
    - \frac{1}{2} \inner{\s_{jk}, \s_{ik}}
    - \frac{1}{2} \inner{\s_{jk}, \s_{0}}
    + \frac{1}{2} \inner{\s_{ik}, \s_{0}},\\ \nonumber
    = & 1 +  \inner{\s_{ik}, \s_0} - \inner{\s_{ij},\s_0} - \inner{\s_{jk}, \s_0}
  }
  and similarly,
  \eq{
    0 \leq& \frac{1}{4}\norm{\s_{ij} + \s_{jk} + \s_{ik} + \s_0}^2, \\
    = & 1 +  \inner{\s_{ik}, \s_0} + \inner{\s_{ij},\s_0} + \inner{\s_{jk}, \s_0}.
    \nonumber
  }
  Using these two inequalities, we can conclude that $\forall i,j,k \in R_p$,
  \eq{\label{eq::triangle-ineq}
    \nonumber
    \abs{\inner{\s_{ij},\s_0} + \inner{\s_{jk}, \s_0}} \leq& 1 + \inner{\s_{ik}, \s_0},\\
    \implies \abs{M_{ij} + M_{jk}} \leq& 1 + M_{ik}.
  }

  Next, we will show that $M$ satisfies
  the cyclic inequalities given in Eq.~\eqref{eq::cyclic-ineq}.
  Recall that $C$ is a chordless cycle $C=\{e_1,e_2, ..., e_N\}$ of $G$,
  and let edges forming $C$ be given as $e_i=(v_{i}, v_{i+1})$ for $i\in [N]$,
  and $v_{N+1} = v_1$.
  Let $F\subset C$ be a set of edges with odd cardinality.
  There is at least one edge belonging $F$.
  We will denote by $e_{i\triangle}$,
  the edge created by joining $v_1$ and $v_i$.
  Note that $e_{2\triangle} = e_1$ and $e_{N\triangle} = e_N$.
  For the simple cycle $C$, by adding the edges
  $\{e_{3\triangle},e_{4\triangle},...,e_{N-1\triangle} \}$
  we have created $N-3$ chords to construct the chordal graph of $C$, where
  $e_i$, $e_{i\triangle}$ and $e_{i+1\triangle}$ form a triangle.

  Let $s_{j} \in \{ -1, +1\}$ be the indicator variable
  for $e_j$'s membership to the set $F$ ($s_{j} = 1$ if $e_j \in F$).
  We have $\prod_{i=1}^Ns_i = (-1)^{N - |F|}$ which
  implies that $s_N = (-1)^{N - |F|}\prod_{i=1}^{N-1}s_i $.
  Finally, we let $s_{i\triangle} = \prod_{j=1}^{i-1}s_j$ for $i\geq 2$
  and observe that $s_{i+1\triangle} = s_{i\triangle}s_{i+1}$. Noticing that
  \eq{
    M(F) - M(C\setminus F) = \sum_{i=1}^N s_i M_{e_i},
  }
  we write the following inequalities
  that are based on the triangle inequalities given in Eq.~\eqref{eq::triangle-ineq},
  \eq{
    &s_{1}M_{e_1} + s_{2}M_{e_2} + s_{3\triangle}M_{e_{3\triangle}} + 1 \geq 0,\\
    \nonumber
    &s_{3}M_{e_3} - s_{3\triangle}M_{e_{3\triangle}} - s_{4\triangle}M_{e_{4\triangle}}  + 1 \geq 0,\\
    \nonumber
    &s_{4}M_{e_4} + s_{4\triangle}M_{e_{4\triangle}} + s_{5\triangle}M_{e_{5\triangle}}  + 1 \geq 0,\\
    \nonumber
    &\hspace{1in}\vdots\hspace{1in} \vdots\\
    \nonumber
    &s_{N-1}M_{e_{N-1}} + (-1)^{N-1}s_{N-1\triangle}M_{e_{N-1\triangle}} + (-1)^{N-1}s_{N\triangle}M_{e_{N\triangle}}  + 1 \geq 0
  }

  By summing these inequalities, we obtain that
  \eq{\label{eq::sum-over}
    \sum_{i=1}^{N-1}s_iM_{e_i} + (-1)^{N-1}s_{N\triangle}M_{e_{N\triangle}} + N -2 \geq 0.
  }
  Since we also have
  $
  s_{N\triangle} = \prod_{i=1}^{N-1}s_i = s_N (-1)^{N-|F|}
  $
  we can write
  \eq{
    (-1)^{N-1}s_{N \triangle} = s_N (-1)^{2N-|F|-1} = s_N
  }
  since $|F|$ is odd. Therefore the inequality in Eq.~\eqref{eq::sum-over}
  reduces to
  \eq{
    \sum_{e \in F}M_e - \sum_{e \in C\setminus F}M_e \geq 2 - N.
  }
  This implies that $M \in \metp$.
  Finally, we invoke the result given in Theorem \ref{thm::barahona}
  and conclude the proof.

\end{proof}

\newpage
\section{Additional Experiments}

\begin{figure}[H]
  \centering
  \includegraphics[width=\linewidth]{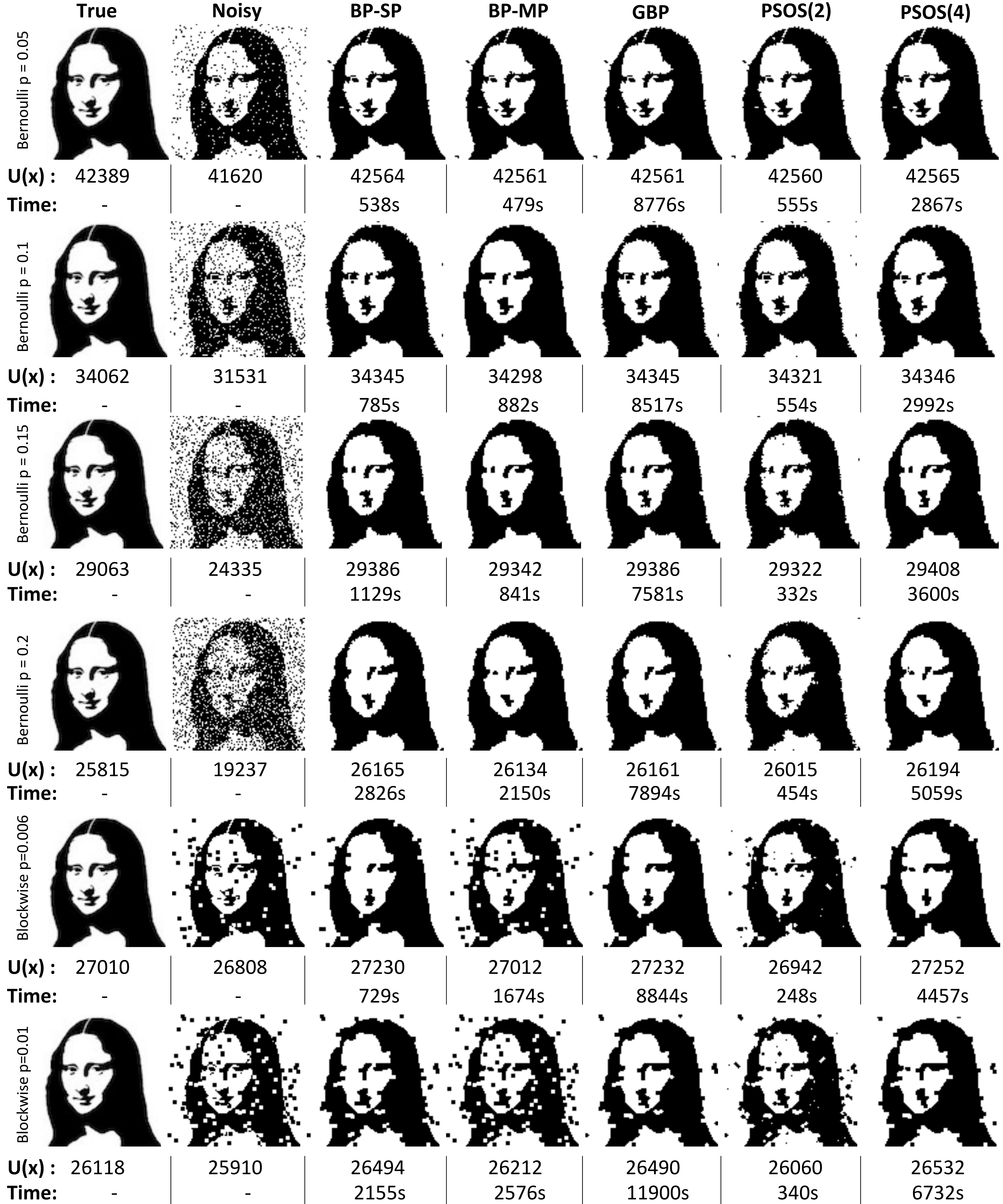}
  \caption{
    Additional denoising experiments of a binary image by maximizing
    the objective function Eq.~(\ref{eq:ObjDenoise}). 
    First 4 rows: i.i.d. Bernoulli error with flip probability
    $p \in \{0.05, 0.1, 0.15, 0.2\}$ with $\theta_0=1.26$.
    Last 2 rows: blockwise noise where each pixel is the center of a
    $3\times 3$ error block independently with
    probability $p \in \{0.006, 0.01\}$ and $\theta_0=1$.
    Final objective value attained by each algorithm along with its run time is
    reported under each image. We observe that PSOS(4) achieves the best objective
    value compared to the other inference algorithms.
    \label{fig::monalisa-many} 
  }
\end{figure}

\newpage

\section{Further Details of the
  Experiments in Section~\ref{sec:SpinGlass}}

\renewcommand{\arraystretch}{1.2}
\begin{table*}[h]\scriptsize
  \caption{
    \label{tab::results} Details of the experiments shown in Figure~\ref{fig::box-plot}.
    We report statistics of run-time and ratio to the best algorithm.
    More specifically, we report the mean and standard deviation of
    the run-time of each algorithm within each experiment (100 replications).
    We also report $\%5/\%10/\%60$ quantiles of the ratio of the objective value achieved by an algorithm
    and the exact optimum for $n\in\{16,25\}$,
    or the best value achieved by any of the 5 algorithms
    for $n\in \{100,400, 900\}$.
  }

  \begin{center}
    \begin{sc}
      \begin{tabular}{l|c|c|c|c|c|c|c|c|c|c}
        \hline
        \hline
        Experiment$\downarrow$\!\!\!& Stats$\downarrow$\!\!\!&PSOS-4 & PSOS-2& GBP & BP-MP& BP-SP\\
        \hline
        \begin{tabular}{@{}c@{}c@{}}$n\!=\!16$ \\ $\theta^v_i\sim\mathcal{U}(\pm 1)$\\$\theta^e_{ij}\sim\mathcal{U}(\pm 1)$\end{tabular}
                                    &\begin{tabular}{@{}c@{}}Time(mean/sd): \\ Ratio(5/10/60\% qt) \end{tabular}
                                    &\begin{tabular}{@{}c@{}} $3.9/.3$\\ $1./1./1.$\end{tabular}
                                    &\begin{tabular}{@{}c@{}} $.2/.0$ \\ $.82/.83/1.$\end{tabular}
                                    &\begin{tabular}{@{}c@{}} $2.8/1.6$ \\ $.91/.92/1.$\end{tabular}
                                    &\begin{tabular}{@{}c@{}} $1./.8$ \\ $.71/.82/1.$\end{tabular}
                                    &\begin{tabular}{@{}c@{}} $1./.4$ \\ $.91/.91/1.$\end{tabular}
        \\
        \hline
        \begin{tabular}{@{}c@{}c@{}}$n\!=\!16$ \\ $\theta^v_i\sim\mathcal{U}(\pm .5)$\\$\theta^e_{ij}\sim\mathcal{U}(\pm 1)$\end{tabular}
                                    &\begin{tabular}{@{}c@{}}Time(mean/sd): \\ Ratio(5/10/60\% qt) \end{tabular}
                                    &\begin{tabular}{@{}c@{}} $4./.3$\\ $1./1./1.$\end{tabular}
                                    &\begin{tabular}{@{}c@{}} $.2/.0$ \\ $.71/ .79/1.$\end{tabular}
                                    &\begin{tabular}{@{}c@{}} $3.7/1.7$ \\ $.78/.83/1.$\end{tabular}
                                    &\begin{tabular}{@{}c@{}} $1.8/.9$ \\ $.57/.65/1.$\end{tabular}
                                    &\begin{tabular}{@{}c@{}} $1.7/.5$ \\ $.65/.74/1.$\end{tabular}
        \\
        \hline
        \begin{tabular}{@{}c@{}c@{}}$n\!=\!16$ \\ $\theta^v_i\sim\normal(0, .01)$\\$\theta^e_{ij}\sim\normal(0,1)$\end{tabular}
                                    &\begin{tabular}{@{}c@{}}Time(mean/sd): \\ Ratio(5/10/60\% qt) \end{tabular}
                                    &\begin{tabular}{@{}c@{}} $4.4/.5$\\ $1./1./1.$\end{tabular}
                                    &\begin{tabular}{@{}c@{}} $.2/.0$ \\  $.72/.83/1.$\end{tabular}
                                    &\begin{tabular}{@{}c@{}} $3.8/2.$ \\ $.67/.76/.97$\end{tabular}
                                    &\begin{tabular}{@{}c@{}} $2.1/.7$ \\ $.41/.5/.98$\end{tabular}
                                    &\begin{tabular}{@{}c@{}} $1.4/.4$ \\ $.52/.69/.95$\end{tabular}
        \\
        \hline
        \begin{tabular}{@{}c@{}c@{}}$n\!=\!16$ \\ $\theta^v_i\sim\normal(0,1)$\\$\theta^e_{ij}\sim\normal(0,1)$\end{tabular}
                                    &\begin{tabular}{@{}c@{}}Time(mean/sd): \\ Ratio(5/10/60\% qt) \end{tabular}
                                    &\begin{tabular}{@{}c@{}} $4./.3$\\   $1./1./1.$\end{tabular}
                                    &\begin{tabular}{@{}c@{}} $.2/.0$ \\  $.78/.87/1.$\end{tabular}
                                    &\begin{tabular}{@{}c@{}} $2.5/1.4$ \\$.8/.86/1.$\end{tabular}
                                    &\begin{tabular}{@{}c@{}} $.9/.33$ \\ $.83/.96/1.$\end{tabular}
                                    &\begin{tabular}{@{}c@{}} $.81/.3$ \\ $.83/.89/1.$\end{tabular}
        \\
        \hline
        \begin{tabular}{@{}c@{}c@{}}$n\!=\!25$ \\ $\theta^v_i\sim\mathcal{U}(\pm 1)$\\$\theta^e_{ij}\sim\mathcal{U}(\pm 1)$\end{tabular}
                                    &\begin{tabular}{@{}c@{}}Time(mean/sd): \\ Ratio(5/10/60\% qt) \end{tabular}
                                    &\begin{tabular}{@{}c@{}} $7.1/.8$\\  $1./1./1.$\end{tabular}
                                    &\begin{tabular}{@{}c@{}} $.3/.0$ \\  $.84/.87/1.$\end{tabular}
                                    &\begin{tabular}{@{}c@{}} $9./3.2$ \\ $.9/.95/1.$\end{tabular}
                                    &\begin{tabular}{@{}c@{}} $2.4/1.6$ \\$.73/.84/1.$\end{tabular}
                                    &\begin{tabular}{@{}c@{}} $1.9/.8$ \\ $.9/.95/1.$\end{tabular}
        \\
        \hline
        \begin{tabular}{@{}c@{}c@{}}$n\!=\!100$ \\ $\theta^v_i\sim\mathcal{U}(\pm 1)$\\$\theta^e_{ij}\sim\mathcal{U}(\pm 1)$\end{tabular}
                                    &\begin{tabular}{@{}c@{}}Time(mean/sd): \\ Ratio(5/10/60\% qt) \end{tabular}
                                    &\begin{tabular}{@{}c@{}} $58./10.$\\ $1./1./1.$\end{tabular}
                                    &\begin{tabular}{@{}c@{}} $1.3/.1$ \\   $.87/ .89/.94$\end{tabular}
                                    &\begin{tabular}{@{}c@{}} $77.7/.4$ \\  $.92/.94/.99$\end{tabular}
                                    &\begin{tabular}{@{}c@{}} $17.7/3.9$ \\ $.85/.87/.96$\end{tabular}
                                    &\begin{tabular}{@{}c@{}} $14.4/2.4$ \\ $.93/.94/.99$\end{tabular}
        \\
        \hline
        \begin{tabular}{@{}c@{}c@{}}$n\!=\!400$ \\ $\theta^v_i\sim\mathcal{U}(\pm 1)$\\$\theta^e_{ij}\sim\mathcal{U}(\pm 1)$\end{tabular}
                                    &\begin{tabular}{@{}c@{}}Time(mean/sd): \\ Ratio(5/10/60\% qt) \end{tabular}
                                    &\begin{tabular}{@{}c@{}} $360.3/83.4$\\ $1./1./1.$\end{tabular}
                                    &\begin{tabular}{@{}c@{}} $5.7/.4$ \\   $.89/.9/.93$\end{tabular}
                                    &\begin{tabular}{@{}c@{}} $386.8/7.$ \\ $.93/.94/.97$\end{tabular}
                                    &\begin{tabular}{@{}c@{}} $83.3/.7$ \\  $.9/.91/.95$\end{tabular}
                                    &\begin{tabular}{@{}c@{}} $69.4/.5$ \\  $.95/.96/.98$\end{tabular}
        \\
        \hline
        \begin{tabular}{@{}c@{}c@{}}$n\!=\!900$ \\ $\theta^v_i\sim\mathcal{U}(\pm 1)$\\$\theta^e_{ij}\sim\mathcal{U}(\pm 1)$\end{tabular}
                                    &\begin{tabular}{@{}c@{}}Time(mean/sd): \\ Ratio(5/10/60\% qt) \end{tabular}
                                    &\begin{tabular}{@{}c@{}} $757.4/108.$\\ $1./1./1.$\end{tabular}
                                    &\begin{tabular}{@{}c@{}} $13.9/1.1$ \\  $.9/ .91/.93$\end{tabular}
                                    &\begin{tabular}{@{}c@{}} $939.8/31.4$ \\$.94/.95/.97$\end{tabular}
                                    &\begin{tabular}{@{}c@{}} $194./1.4$ \\  $.91/.92/.95$\end{tabular}
                                    &\begin{tabular}{@{}c@{}} $161.8/1.3$ \\ $.95/.96/.97$\end{tabular}
                                    \\
        \hline
      \end{tabular}
    \end{sc}
  \end{center}
  \vskip -0.1in
\end{table*}

\end{document}